\documentclass{article}

\PassOptionsToPackage{numbers, compress}{natbib}
\usepackage[preprint]{nips_2018}

\usepackage[utf8]{inputenc} 
\usepackage[T1]{fontenc}    
\usepackage{hyperref}       
\usepackage{url}            
\usepackage{booktabs}       
\usepackage{amsfonts}       
\usepackage{nicefrac}       
\usepackage{microtype}      
\usepackage{amsmath}
\usepackage{mathrsfs}
\usepackage{color}
\usepackage{csquotes}

\usepackage{graphicx}
\graphicspath{ {images/} }
\usepackage{scrextend}
\usepackage{booktabs}
\usepackage{float}
\usepackage{listings}

\usepackage{subcaption}

\usepackage{pgfplots}
\pgfmathsetseed{1}
\usepackage{tikz}
\usetikzlibrary{arrows} 
\usetikzlibrary{calc} 
\usetikzlibrary{math} 
\usetikzlibrary{decorations}
\usetikzlibrary{decorations.pathmorphing}
\usepackage{tkz-euclide}
\usetkzobj{all}

\usepackage{amsthm}
\newtheorem{theorem}{Theorem}

\newtheorem{definition}{Definition}

\usepackage{algorithm}
\usepackage{algorithmicx}
\usepackage[noend]{algpseudocode}

\usepackage[font=small,labelfont=bf]{caption}

\usepackage{titlesec}
\titlespacing*{\section}{0pt}{1em}{1em}

\makeatletter
\def\BState{\State\hskip-\ALG@thistlm}
\makeatother

\tikzstyle{block} = [draw,fill=blue!20,minimum size=2em]
 
\tikzstyle{branch}=[fill,shape=circle,minimum size=3pt,inner sep=0pt]

\definecolor{xuxi}{rgb}{0.59, 0.0, 0.09}
\definecolor{josh}{rgb}{0.0, 0.42, 0.24}
\definecolor{peng}{rgb}{0.0, 0.0, 1.0} 
\definecolor{new}{rgb}{1.0, 0.49, 0.0}
\definecolor{old}{rgb}{0.25, 0.25, 0.25}
\definecolor{proposed}{rgb}{1.0, 0.49, 0.0}
\definecolor{accepted}{rgb}{0.0, 0.0, 0.0} 

\newcommand{\algo}{\texttt{Exact}}

\newcommand{\bb}{\mathcal{B}}
\newcommand{\A}{\mathcal{A}}

\title{Fast Online Exact Solutions for Deterministic MDPs with Sparse Rewards}

\author{
  Joshua R. Bertram    \quad    Xuxi Yang    \quad    Peng Wei\\
  Iowa State University\\
  Ames, IA 50011 \\
  \texttt{\{bertram1, xuxiyang, pwei\}@iastate.edu} \\
}

\begin{document}

\maketitle

\begin{abstract}
Markov Decision Processes (MDPs) are a mathematical framework for modeling sequential decision making under uncertainty. The classical approaches for solving MDPs are well known and have been widely studied, some of which rely on approximation techniques to solve MDPs with large state space and/or action space.  However, most of these classical solution approaches and their approximation techniques still take much computation time to converge and usually must be re-computed if the reward function is changed. This paper introduces a novel alternative approach for exactly and efficiently solving deterministic, continuous MDPs with sparse reward sources.  When the environment is such that the "distance" between states can be determined in constant time, e.g. grid world, our algorithm offers $O( |R|^2 \times |A|^2  \times |S|)$, where $|R|$ is the number of reward sources, $|A|$ is the number of actions, and $|S|$ is the number of states.  Memory complexity for the algorithm is $O( |S| + |R| \times |A|)$. This new approach opens new avenues for boosting computational performance for certain classes of MDPs and is of tremendous value for MDP applications such as robotics and unmanned systems. This paper describes the algorithm and presents numerical experiment results to demonstrate its powerful computational performance. We also provide rigorous mathematical description of the approach.
\end{abstract}

\section{Introduction}

	A Markov Decision Process (MDP) is a method for modeling an agent's decision making process where the agent transitions from the current state to a next state by taking an action.  The value of the next state (and hence the action) is determined via a reward function that describes which states within the state space provide reward (or penalty.)  The solution to an MDP, termed the optimal policy, describes the optimal action to take from any state and maximizes \textcolor{accepted}{the expected cumulative} reward.  Following the optimal policy from any initial state will result in an optimal \textcolor{accepted}{sequence of actions} through the state space for maximizing \textcolor{accepted}{cumulative} reward.
    
	While MDPs are a powerful and general technique, they have difficulty scaling to large state or action spaces.  The original dynamic programming algorithms \textcolor{accepted}{such as value iteration and policy iteration} used to solve MDPs required a full representation of the state (or state-action) space to be stored in memory, rendering many complex MDP problems intractable.  Furthermore, most MDP algorithms are designed for stationary MDPs where the environment and the reward system are fixed, and require the solution to be recomputed if the reward changes.  Given the computation time required to solve any sizable MDP, this often precludes MDPs from being used in non-stationary reward systems unless the MDP can be reformulated to capture the changes in the reward function as part of the state itself.
    
    This paper presents an alternate approach \textcolor{accepted}{for} efficiently solving deterministic continuous MDPs exactly and presents a rigorous mathematical proof for single reward sources.  This paper also proposes an algorithm for solving deterministic continuous MDPs with multiple reward sources exactly and presents numerical experiment results to validate the algorithm and demonstrate its powerful computational performance.  We also demonstrate that the improved performance of the algorithm allows MDPs with changing reward functions to be computed online within execution times suitable for real-time systems, which will be of tremendous value for MDP applications in robotics, unmanned systems, and other domains with structured state spaces and sparse reward functions.  Our algorithm performance is compared with the standard value iteration algorithm to \textcolor{accepted}{demonstrate the substantial performance improvement}.
    
    In this paper we restrict ourselves to deterministic continuous MDPs with non-negative real rewards to constrain the length of the paper and will attempt to generalize to a broader class of MDPs in future work.
    
\section{Background}
	
\subsection{Markov Decision Processes}
	An MDP is described as having a state space $S$, \textcolor{accepted}{an} action space $A$, \textcolor{accepted}{a} transition function $T$, and \textcolor{accepted}{a} reward function $R$.  From a given state $s \in S$, an action $a \in A$ results in a next state $s' \in S$ with a probability $T( s, a, s' )$.   \textcolor{accepted}{Reward functions} are normally defined as immediate reward $R( s, a )$ obtained after taking action $a$ from state $s$ (though in this paper only consider reward $R( s )$ that depends only on the state.) MDPs assume that the Markov \textcolor{accepted}{property} holds for the problem, which states that the future state $s'$ depends only on the current state $s$ and the action taken $a$.  \textcolor{accepted}{ A \emph{policy} $\pi$ describes an action $a$ that should be performed at each state $s$; an optimal policy $\pi^*$ is a policy that generates maximum cumulative reward for every state.  Solving an MDP is equivalent to finding the optimal policy $\pi^*$.}

	In this initial paper, we restrict ourselves to deterministic continuous MDPs; for ease of illustration and without loss of generality, we use a two dimensional grid world as the example environment.  In a \emph{deterministic} MDP, the specified action (e.g., up, down, left, right) occurs without failure or error with a probability of $1.0$.  A \emph{continuous} MDP has no natural end of the sequence once it reaches a goal state; it passes through the goal state and continues to accrue reward without end. \textcolor{accepted}{In a grid world, the state corresponds to a cell of the grid. At each state, the agent has four actions to choose from: up, down, left, and right, except for boundary state, where the agent is not allowed to take actions which lead the agent to the boudary. Thus in boundary state, the agent has two or three actions to choose from.}
    
    Note that in this paper, we will use the following convention to differentiate between the state at time $t$ with $s^{(t)}$ with a superscript and parentheses and a particular state $s_k \in S$ with a subscript.  Thus, the state $s_k$ at time $t$ would be denoted $s_k^{(t)}$.  Similarly, an action $a_k$ and reward $r_k$ at time $t$ would be denoted as $a_k^{(t)}$ and $r_k^{(t)}$ respectively.  A superscript by itself indicates raising to a power, as in the discount factor $\gamma$ being raised to the power of $t$ in $\gamma^t$.  A state $s$ may refer to either a state $s \in S$ or a "current" state, where a state $s'$ always refers to a next state.

\subsection{Value Iteration}

    One of the most fundamental approaches to solving an MDP is \emph{value iteration}.  Value iteration is a dynamic programming approach to solving an MDP which iteratively determines the value of each state.  \textcolor{accepted}{ 	In an infinite horizon problem with a discount factor of $0.0 < \gamma < 1.0$ the expected cumulative reward (or value) at time step $t$ associated with a sequence of immediate rewards $r_t$ is:
	\begin{equation}
	V(s) = \sum_{t=0}^{\infty} \gamma^t r^{(t)}(s^{(t)},a^{(t)})
	\end{equation}
More generally, this is expressed recursively using the Bellman equation\cite{bellman1957dynamic}.
	\begin{equation}
	V_{n}(s^{(t)}) = \max_{a}  \Big[ r(s^{(t)},a) + \gamma \sum_{s^{(t+1)}} T(s^{(t+1)}| s^{(t)},a)  V_{n-1}(s^{(t+1)}) \Big],
	\end{equation}
where $n$ is the current iteration of the value iteration algorithm.  Value iteration obtains the optimal value when the policy and value function become stationary with respect to the Bellman operator $L$ satisfying the equation $V^* = L V^*$. (See Chapter 1 of \cite{yellowbook} for more information on this important topic as well as Bellman's original treatment \cite{bellman1957dynamic}.)
}

Examining the runtime complexity of value iteration, from \cite{yellowbook, papadimitriou1987complexity}, every iteration of the value iteration algorithm takes $O( |A| \times |S|^2 )$, and the overall maximum number of iterations needed by the algorithm is polynomial in $|S|$, $|A|$, and $\frac{1}{1-\gamma}\log{\frac{1}{1-\gamma}}$.
\section{Related Work}
	Many \textcolor{accepted}{classical} approaches to solving MDPs exist (see the excellent texts \cite{suttonbarto}, \cite{yellowbook}, \cite{DMUBook}, and \cite{russellnorvig}).  \textcolor{accepted}{The value iteration algorithm requires that the entire state space be held in memory as a table; exceedingly large state and action spaces will exhaust the memory available to even large, powerful computers. Thus at some level of complexity, value iteration becomes intractable due to the explosion in state-action space.}  
    
    Some algorithms address the explosion in state-action space size by taking advantage of structure of the problem to more compactly represent the MDP \cite{schuurmans2002direct}, \cite{guestrin2003efficient} which can lead to performance improvements.  The other major thrust is to create methods that use value function approximations or estimation techniques to avoid having to maintain a table in memory of each state \cite{powell2007approximate}.  Attempts to scale to high dimensional state spaces with neural nets \cite{DQNatari} and Monte-Carlo Tree Search (MCTS) \cite{kocsis2006bandit} represent recent attempts to deal with state space explosion, and MCTS especially has shown promising results for the game of Go \cite{silver2016mastering, silver2017mastering}.  
    
Attempts have been made to improve the performance of value iteration.  Asynchronous value iteration is a variation of value iteration that processes only certain states during each iteration which improves memory usage and can converge more quickly than value iteration in some cases \cite{DMUBook}.  Prioritized sweeping is another strategy that orders the processing of the states via some metric and after updating backpropagates to predecessors of the processed state \cite{moore1993prioritized, wingate2005prioritization}. An excellent summary of research in this area is provided in \cite{de2012new}.  Of particular interest to this paper, \cite{mcmahan2005fast} examines stochastic shortest path problems using an approach based off Dijkstra's algorithm (also discussing deterministic MDPs), building on work in \cite{bertsekas1995dynamic}, where they show that deterministic MDPs can reduce to Dijkstra's algorithm which has some performance benefits over value iteration.  In \cite{dai2007prioritizing}, methods are discussed that eliminate a priority queue typically required.  Of particular note is a backwards value iteration algorithm which computes value iteration from a terminating goal state, considering the problem in terms of states in which the goal state is reachable and working backwards from there.  While they do eliminate the overhead of a priority queue, but retain  a FIFO queue.  They similarly propose a forward value iteration algorithm that considers states that are reachable from the initial state and work forward from there, which they point out is equivalent to the LAO* algorithm in \cite{hansen2001lao}.

Varying the environment has been studied in \cite{szita2002varepsilon} with the restriction that the changes to the transition function $T$ are bounded by some small value $\epsilon$.  In \cite{yu2008markov}, reward is allowed to vary arbitrarily between time steps in a regret-based framework that relies on solving a linear program at each step.  Both the environment and rewards are varied in \cite{yu2009online} using a robust dynamic programming method which also ultimately relies on linear programming at each time step.  In \cite{even2005experts}, reward is allowed to change arbitrarily at each time step (possibly an adversarial manner) in a stochastic setting where $N$ black-box "experts" are provided; convergence bounds with respect to a fixed horizon and expected regret are provided on resulting policy changes, and performance is shown to be polynomial with respect to the $N$ experts and $|A|$ actions, though it relies on the existence of the expert algorithms (which are not within the scope of the paper itself).  In \cite{van2017hybrid}, reward functions are decomposed into simpler MDPs, each are solved with an neural net based approach similar to DQN, and the results of the simpler MDP Q-value functions are aggregated into a resulting global Q-value function, but the approach does not lead to a more fundamental understanding of how the value function is composed from the smaller MDPs.  Time-dependent MDPs (TMDPs) in \cite{boyan2001exact} are used to calculate MDPs with a continuous time dimension, claiming an exact solution in terms of piece-wise linear time steps but still relies on value iteration to approximate the true (exact) solution.  
   
To the authors' knowledge, no existing research provides direct solutions of the exact value function for this (or any other) class of MDPs in the manner described in this paper.

\section{Methodology}

    In this paper we describe an MDP in terms of a graph.  As will become clear later, we adopt this convention in order to take advantage of properties that will emerge to arrive at a method to calculate the exact solution to MDPs with reward functions $R(s)$ that depend only on state $s$.

\subsection{\textcolor{accepted}{MDP Transition Graphs}}

	An MDP can theoretically allow a transition from any state $s$ to any other state $s'$ by action $a$, which is defined by the transition function $T(s, a, s')$.  A zero value for a given $s$, $a$, $s'$ means no transition is possible, otherwise a probability from $(0, 1]$ is given and the state $s$ and $s'$ \textcolor{accepted}{is defined in this paper as \emph{connected}}.  The probabilities of transition from any state $s$ to all possible next states (including state $s$ itself) must total $1.0$.
	
	A \emph{transition graph} for a deterministic MDP can be defined where each node of the graph is a state $s$ and each edge of the graph is a possible action $a$.  The transition graph is a directed graph which may be cyclic.  (Note that this is just a graphical representation of the information contained in the transition function $T$.)  See Appendix 1 for an example transition graph.  Also note that a similar description is provided in \cite{papadimitriou1987complexity} for deterministic MDPs.   
    
For a deterministic transition graph, the \emph{distance} is defined as the minimum positive number of actions (or transitions) needed to move from a given state $s_0$ to a desired state $s_k$. 

Formally, suppose an MDP has $n$ states $\mathcal{S} = \{s_1, s_2, \cdots, s_n \}$. At each state, there are $m$ actions to choose: $\A = \{ a_1, a_2, \cdots, a_m \}$.  At time $t$, the state is denoted $s^{(t)} \in \mathcal{S}$ and action $a^{(t)} \in \A$.  Since this MDP is deterministic, the next state given current state and current action can be denoted as $s^{(t+1)} = T(s^{(t)}, a^{(t)})$, where $s^{(t)}$ and $a^{(t)}$ are the current state and current action, and the mapping $T: \mathcal{S} \times \A \rightarrow \mathcal{S}$ is the next state $s^{(t+1)}$ according to current state and action.  

Suppose the initial state is $s^{(0)}$, after taking action $a^{(0)}$, the next state is $s^{(1)} = T(s^{(0)}, a^{(0)})$.  After taking another action $a^{(1)}$, the third state will be $T(T(s^{(0)}, a^{(0)}), a^{(1)})$. For convenience, we denote this as $T(s^{(0)}, a^{(0)}, a^{(1)})$. More generally, if the initial state is $s^{(0)}$, after taking sequential actions $a^{(0)}, a^{(1)}, a^{(2)}, \cdots, a^{(t)}$, the agent will be at state $T(s^{(0)}, a^{(0)}, a^{(1)}, a^{(2)}, \cdots, a^{(t)})$.

\begin{definition}
For a deterministic MDP with finite states, if from state $s_i$, after taking finite actions, the agent can reach state $s_k$, then the distance from $s_i$ to $s_k$ is defined as: 
\begin{equation}
  \delta (s_i, s_k) = \min_t \{ t| T(s_{i}, a^{(1)}, a^{(2)}, \cdots, a^{(t)}) = s_k  \}
  \label{eq:conn_dist}
\end{equation}

If from state $s_i$, no matter what sequence of actions the agent takes, it cannot reach state $s_k$, then the distance from $s_i$ to $s_k$ is defined to be:

\begin{equation}
  \delta(s_i, s_k) = \infty
  \label{eq:inf_conn}
\end{equation}

Finally, we define the distance from a state to itself $\delta (s,s) = 0$ for any $s \in \mathcal{S}$. 
\end{definition}

Note that for a two dimensional grid world MDP, the distance from one state to another state is just the Manhattan distance with respect to the Cartesian coordinate of the grid cells.

\begin{definition}
  An MDP problem is said to be a fully connected MDP if all states can be reached from all other states:  $\forall s, s' \in S, \delta(s, s') < \infty$.
\end{definition}

By the definition of fully connected MDP, we wish to examine MDPs in which the agent can arrive at any state from any given initial state (that is, all states are potentially reachable through some set of actions.)  

\subsection{Exact Solutions for Single Reward Sources}

Given the definition of \textcolor{accepted}{MDP transition graph}, we now describe the exact solution to deterministic continuous MDPs for a single reward source.  We first define the concept of a cycle which occurs in continuous MDPs and derive the exact value function.

\begin{definition}
The \emph{cycle} of a state $s$, which denoted as $\mathcal{C}(s)$, is an ordered sequence of states: $s^{(1)}, s^{(2)}, \cdots, s^{(t)}$ where the states in this sequence satisfy the following condition:

There exists a sequence of action $a^{(1)}, a^{(2)}, \cdots, a^{(t+1)}$ such that:
\begin{gather*}
  T(s, a^{(1)}) = s^{(1)}	\\
  T(s, a^{(1)}, a^{(2)}) = s^{(2)}	\\
  \cdots	\\
  T(s, a^{(1)}, a^{(2)}, \cdots, a^{(t)}, a^{(t+1)}) = s
\end{gather*}

The length of the cycle, $d( \mathcal{C}(s) )$ is the number of actions in the sequence ($t+1$) that causes a return to $s$.
\end{definition}

Note that if a state $s$ has more than one cycle there always exists one cycle with finite distance.  If there exists an action $a \in \A$ such that $T(s, a) = s$, this is also a cycle with distance 1.  Note that a state $s$ may have no cycle.  Note also that a state $s$ can have more than one cycles and that the states in these cycles do not need to be distinct.  (Some states of a given cycle may be shared with other cycles.)

\begin{definition}
Suppose a state $s$ has $p$ cycles $\mathcal{C}^1(s), \cdots, \mathcal{C}^p(s)$, \textcolor{accepted}{where $p$ can be finite or infinite.} The \emph{minimum cycle} of state $s$, which denoted as $\mathcal{C}^*(s)$, is a cycle with minimum distance:\\
\begin{equation}
  \mathcal{C}^* = \{ \mathcal{C}^i | d(\mathcal{C}^i) \leq d(\mathcal{C}^j), \forall j \in \{ 1, \cdots, p \} \}
\end{equation}
\end{definition}

Note that a state $s$ can have no minimum cycle, if and only if the state $s$ has no cycles. And a state $s$ can also have more than one minimum cycle when there are more than one cycles having same minimum distance among all the cycles.

We denote the distance of the minimum cycle of state $s$ as $\phi(s)$. 

We now describe how to calculate the value function given this definition of a minimum cycle.  

\begin{theorem}
Every deterministic continuous \textcolor{accepted}{fully connected} MDP with a single reward source has a minimum cycle. 
\label{thm:min_cycle_exists}
\end{theorem}

The proof is provided in Appendix 3.

\begin{theorem}
For the deterministic continuous \textcolor{accepted}{fully connected} MDP model defined above, suppose it has reward function as follows:
\begin{equation}
r(s) = \begin{cases}
          r_g>0 & \mbox{if } s=s_g \\
    	  0  & \mbox{otherwise}.
	    \end{cases}
\label{eq:reward_cont}
\end{equation}

where $s, s_g \in S$ is the state where reward $r_g$ is collected.

Then the value function of this MDP with discount factor $0 < \gamma <1$ has the form
\begin{equation}
\begin{split}
V(s) &= \gamma^{\delta(s,s_g)} \times \frac{r_g}{1 - \gamma^{\phi(s_g)}} \\
\label{eq:cont_cycle}
\end{split}
\end{equation}
\label{thm:val_function}
\end{theorem}

The proof is provided in Appendix 3.

\subsection{Exact Solution for Multiple Reward Sources}

We move now to discuss how to find the value function when multiple reward sources are present.  First, we define multiple reward sources as having $N>0$ positive rewards $R = \{ r_1, ... r_N \}$, where we will refer to the number of rewards as $|R|$.  We introduce some terms to help us describe the algorithm.  Informally, we describe a state where the value function increases due to acquiring a reward as a \emph{peak}, and each reward will generate a peak.  

If we assume that there is a peak value $v_g$ at state $s_g$, then we will term the operation of calculating the whole value function from a peak as \emph{propagating reward} and will denote this for reward $r_g$ at state $s_g$ generally as:
\begin{equation}
    \mathcal{P}_g(s) = \gamma^{\delta(s,s_g)} \times v_g
    \label{eq:propagate}
\end{equation}
where $\delta(s,s_g)$ is the distance from $s$ to $s_g$.  Note that this simply corresponds to the discounted future reward from state $s$ with respect to reward $r_g$, but is a convenient notational shorthand.  

\begin{definition}
We use the term \emph{baseline peak}, $\mathcal{B}^i$, to describe the a single reward source $r_i$ located at state $s_i$ which is collected infinitely but has no other rewards in its minimum cycle.  When the context is clear or we are speaking generally of a baseline peak, we may drop the $i$ and denote the baseline peak as $\mathcal{B}$.  \textcolor{accepted}{The value at the baseline peak is:
\begin{equation*}
  \mathcal{B}^i = \frac{r_i}{1 - \gamma^{\phi(s_i)}}
\end{equation*}}

\textcolor{accepted}{
The value function for the baseline peak $\mathcal{B}^i$ is a mapping $\mathcal{P}: \mathcal{S} \rightarrow \mathbb{R}$:
\begin{equation}
  \mathcal{P}_{\mathcal{B}^i}(s) = \gamma^{\delta(s,s_i)} \times \frac{r_i}{1 - \gamma^{\phi(s_i)}}
\end{equation}}
\end{definition}

\begin{definition}
We use the term \emph{combined peak} to describe a primary reward source at state $s_p$ which is collected infinitely and has a secondary reward source at state $s_s$ within the primary state's minimum cycle.  We denote the value of the combined peak at $s_p$ as $\Gamma^{p,s}$ (or $\Gamma^p$ or even $\Gamma$ when the context is clear.)  The value function for a combined peak is 
\begin{equation}
\mathcal{P}_{\Gamma^{p,s}}(s) = \mathcal{P}_{\mathcal{B}^p}(s) + \mathcal{P}_{\mathcal{B}^s}(s)
\end{equation}
\end{definition}


\textcolor{accepted}{For rewards that are collected just once, we refer to the increase to the value function of collecting this reward as a \emph{delta peak}.  This represents a "bump" in the value function at the state where the reward is collected, which is propagated outward.}

\begin{definition}
A delta peak for a reward $r_i$ is calculated by adding the reward $r_i$ at state $s_i$ to some pre-existing value function $\mathcal{P}^j(s)$ formed by propagation. At $s_i$, the value of the delta peak is $\Delta^i = r_i + \mathcal{P}^j( s_i )$.  The value function for the delta peak is formed by propagation:
\begin{equation}
  \mathcal{P}_{\Delta^i}(s) = \gamma^{\delta(s,s_i)} \times \Delta^i
\end{equation}
\end{definition}

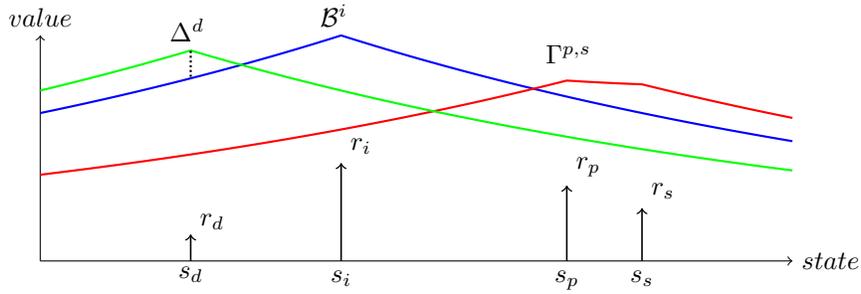
\begin{figure}[H]
\centering
\begin{tikzpicture}[decoration={markings, mark=at position 1.0 with {\arrow{>}};}]
      \tkzDefPoint(2,.05){A}
      \tkzDefPoint(2,.3){B}
      \tkzDrawLines[postaction={decorate}](A,B)
      \tkzLabelPoint[above right](B){$r_d$}
      \tkzLabelPoint[below](A){$s_d$}
      
      \tkzDefPoint(4,.00){A}
      \tkzDefPoint(4,1.3){B}
      \tkzDrawSegments[postaction={decorate}](A,B)
      \tkzLabelPoint[above right](B){$r_i$}
      \tkzLabelPoint[below](A){$s_i$}
      
      \tkzDefPoint(7,.00){A}
      \tkzDefPoint(7,1.0){B}
      \tkzDrawSegments[postaction={decorate}](A,B)
      \tkzLabelPoint[above right](B){$r_p$}
      \tkzLabelPoint[below](A){$s_p$}
      
      \tkzDefPoint(8,.00){A}
      \tkzDefPoint(8,0.7){B}
      \tkzDrawSegments[postaction={decorate}](A,B)
      \tkzLabelPoint[above right](B){$r_s$}
      \tkzLabelPoint[below](A){$s_s$}

      \draw[->] (0 ,0) -- (10,0) node[right] {$state$};
      \draw[->] (0 ,0) -- (0,3) node[above] {$value$};
      \draw[thick,domain=0:10,samples=1000,variable=\x,blue] plot ({\x},{3*0.9^(abs(\x-4))});
      \draw[thick,domain=0:10,samples=1000,variable=\x,red] plot ({\x},{1.5*0.9^(abs(\x-7))+1*0.9^(abs(\x-8))});
      \draw[thick,domain=0:10,samples=1000,variable=\x,green] plot ({\x},{2.8*0.9^(abs(\x-2))});
      \draw[thick,densely dotted,domain=2.45:2.8,smooth,variable=\y,black]  plot (2,{\y});
      \tkzLabelPoint[above right](3.6,3  ){$\mathcal{B}^i$}
      \tkzLabelPoint[above right](6.6,2.5){$\Gamma^{p,s}$}
      \tkzLabelPoint[above right](1.6,2.8){$\Delta^d$}
      
\end{tikzpicture}
\caption{Illustration of baseline (blue), combined baseline (red), and delta baseline (green)}
\end{figure}

\subsubsection{Algorithm}

The algorithm, which we have named \algo, is designed to maintain a sorted list of all valid possible peaks at any time.  Each iteration, it selects the maximum peak from the list and this peak is considered \emph{processed}.  The processed peak has at least one affected reward (combined peaks have more than one); all peaks that are composed from the affected rewards are removed from the list.  

Baseline peaks and combined peaks rely only on the value of the reward, and are therefore pre-calculated and added to the reward list before the first iteration.  Delta rewards however depend on the value function at each iteration and are therefore calculated at the beginning of each iteration.  Processing continues until the list of possible peaks is empty.

\begin{algorithm}[H]
\caption{\algo}\label{algorithm}
\begin{algorithmic}[1]
\Procedure{\algo}{$\textit{rewardSources}$}
\State $\textit{valueFunction} \gets \text{preallocate array of zeros} $ 
\State $\textit{processedPeaks} \gets \text{empty list}$
\State $\textit{sortedPeaks} \gets \textit{PrecomputePeaks( rewardSources )}$
\While {$\textit{sortedPeaks} \text{ is not empty}$}
	\State $\textit{deltaPeaks} \gets \textit{ComputeDeltas( valueFunction )}$
	\State $\textit{sortedPeaks} \gets \textit{PruneInvalidPeaks()}$
	\State $\textit{maxPeak} \gets \textit{max( [ sortedPeaks, deltaPeaks ] )}$
	\State $\textit{sortedPeaks} \gets \textit{RemoveAffectedPeaks( maxPeak )}$
	\State $\textit{valueFunction} \gets \textit{UpdateValueFunction()}$
\EndWhile
\Return $\textit{valueFunction}$
\EndProcedure
\end{algorithmic}
\end{algorithm}

Line 2 initializes memory to hold the value function.  Line 3 initializes an empty list to track which peaks have been processed by the algorithm.  Line 4 pre-computes baseline peaks and combined peaks based off a list of reward sources and stores them in the form of a sorted list, sorted by value of each peak.  Lines 5-10 continue until we have exhausted the potential peaks and each iteration of the loop whittles away at the list of possible peaks.  Line 6 computes delta peaks for any remaining reward sources utilizing the value function that has been computed so far.  Line 7 removes any peaks that have become invalid.  Line 8 selects the peak with maximum value.  Line 9 removes any other potential peaks in the list that are affected by selecting the peak with maximum value.  (For example, a combined peak with states 3 and 4 are selected.  The baselines for states 3 and 4 would then be removed.)  Line 10 then updates the value function based off the newly selected peak with maximum value.  Complete pseudo code is provided in Appendix 2.

In addition to the proof provided in the appendix, the algorithm was additionally verified by generating test scenarios where randomly sized grid worlds with randomly generated rewards.  The number of rewards varied between $1$ and $|S|$.  The MDP was solved with value iteration, and then the result was used to verify the operation of our algorithm.

Performance of the algorithm is $O( |R|^2 \times |A|^2  \times |S|)$, where $|R|$ is the number of reward sources, $|A|$ is the number of actions, and $|S|$ is the number of states.  Memory complexity for the algorithm is $O( |S| + |R| \times |A|)$.  

\section{Experiments}

\begin{figure}[H]
\centering
\begin{subfigure}{.33\textwidth}
  \centering
  \includegraphics[width=1\linewidth]{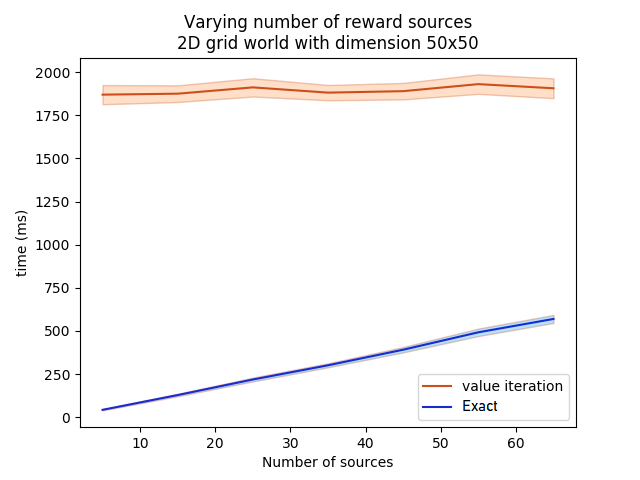}
  \caption{Varying number of reward sources}
  \label{fig:rewards}
\end{subfigure}%
\begin{subfigure}{.33\textwidth}
  \centering
  \includegraphics[width=1\linewidth]{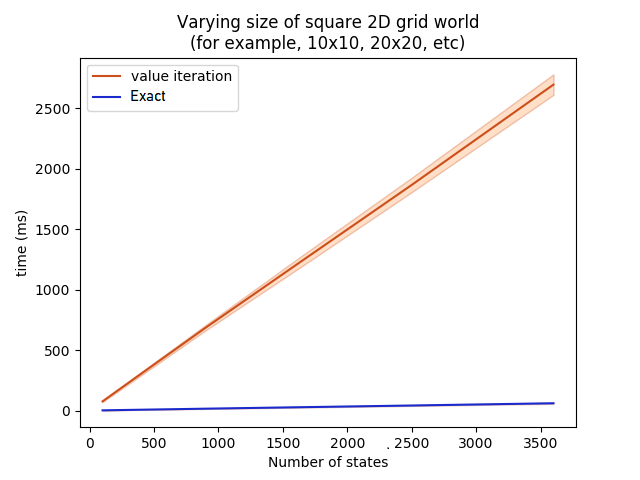}
  \caption{Varying number of states}
  \label{fig:states}
\end{subfigure}
\begin{subfigure}{.33\textwidth}
  \centering
  \includegraphics[width=1\linewidth]{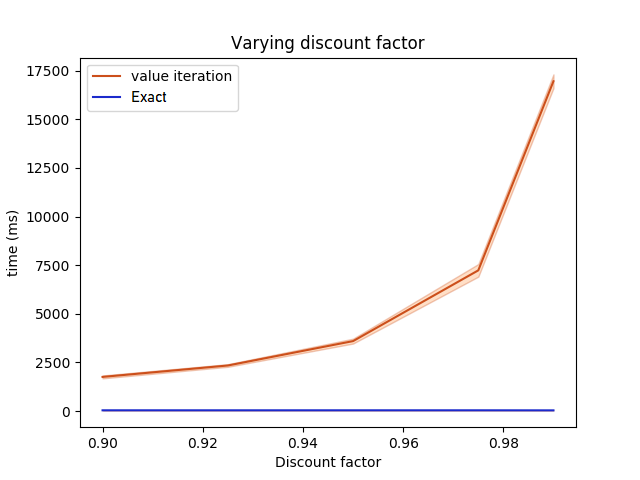}
  \caption{Varying discount factor}
  \label{fig:discount}
\end{subfigure}%
\caption{Experimental results showing the improved performance of the algorithm as compared to value iteration when varying the number of reward sources, varying the number of states, and varying the discount factor.}
\label{fig:test}
\end{figure}

Figure \ref{fig:rewards} shows the effects of varying the number of reward sources on the performance of the algorithm.  For this result, a 50x50 grid world was used.  The $x$-axis shows the number of reward sources used for a test configuration and the $y$-axis shows the length of time required to solve the MDP.  For each test configuration, 1000 randomly generated configurations were created for the number of reward sources specified in the test configuration with reward values ranging from 1 to 10.  For each generated configuration, value iteration and our proposed algorithm (named \algo) were run to obtain performance measurements.  As an additional check, the exact solution calculated by this algorithm was compared to the value iteration result to ensure they produced the same result (within a tolerance due to value iteration approximating the exact solution with the use of a bellman residual as a terminating condition.)  In the plot, the bold line is the average and the colored envelope shows the standard deviation for each test configuration.

The figure shows that as the number of reward sources increases, value iteration remains invariant of the number of reward sources.  For the algorithm proposed in this paper, for small numbers of reward sources the algorithm clearly outperforms value iteration.  As the number of reward sources increases, however, we expect an intersection point will occur and value iteration will begin to perform better.

Figure \ref{fig:states}  shows the effects of varying the size of the state space on the performance of the algorithm.  For this a fixed number of reward sources (5) were used, and only the size of the state space was varied (by making the grid world larger).  The $x$-axis shows the number of states in the grid world (e.g., $10\times10=100, 50\times50=2500$) and the $y$-axis shows length of time required to solve the MDP.  For each grid world size, 1000 randomly generated reward configurations with the fixed number of reward sources were generated.  The results show that value iteration quickly increases in execution time when the state space increases whereas the algorithm proposed in this paper increases a much slower rate.

Figure \ref{fig:discount} shows the effects of varying the discount factor on the performance of the algorithm.  For this test, a fixed number of reward sources (5) and state space size (50x50) were used, and only the discount factor was varied.  The $x$-axis shows the discount factor and the $y$-axis shows the length of time required to solve the MDP.  For each discount factor, 1000 randomly generated reward configurations with the fixed discount factor were generated.  The results show that value iteration increases apparently exponentially with the discount factor, whereas the algorithm proposed in this paper is invariant to the discount factor.  This follows from the exact calculation of the value based off the distance, where the discount factor is simply a constant that is used in the calculation.

All tests were performed on a high-end "gaming class" Alienware  laptop with a quad-core Intel i7 running at 4.4 GHz with 32GB RAM without using any GPU hardware acceleration (i.e., CPU only).  All code is single threaded, python only and no special optimization libraries other than numpy were used (for example, the python numba library was not used to accelerate numpy calculations.)  Both value iteration and the proposed algorithm use numpy.  The results presented here are meant to most fairly present the performance differences between the algorithms, thus further optimizations should yield improved performance beyond what is presented here.

\section{Conclusion}

In this paper we have presented a novel approach to solving deterministic continuous MDPs exactly which we believe is the first example of this technique.  This new algorithm's computational speed greatly exceeds that of value iteration for sparse reward sources and, furthermore, is invariant to the discount factor.  The complexity of the algorithm is $O( |R|^2 \times |A|^2  \times |S|)$, where $|R|$ is the number of reward sources, $|A|$ is the number of actions, and $|S|$ is the number of states.  Memory complexity for the algorithm is $O( |S| + |R| \times |A|)$.

For environments where the model is known and the dynamics are deterministic, this method will have immediate applicability.  It may also be of use in certain robotics and unmanned vehicle applications where the state uncertainty is negligible and real-time constraints must be met.  Additionally this method allows the position and magnitude of rewards in the state space to change arbitrarily between time steps without affecting performance.

Future work will consist of handling negative rewards and extending to more general classes of MDPs such as stochastic MDPs.  We also plan to test performance against other classes of algorithms and show how the algorithm can be applied in various problem domains.

\bibliographystyle{unsrt}
\bibliography{refs}

\newpage
\section{Appendix 1: Example Transition Graph}

A sample transition matrix for a \textcolor{accepted}{4-state, 2-action} deterministic problem might be: 

\begin{center}
    \begin{tabular}{ | c | c | c | c | c | c |}
    \hline
          & action     & $s_0$ &   $s_1$ &   $s_2$ &  $s_3$ \\ \hline
    $s_0$ & $a_0$      & 0.0   &     1.0 &     0.0 &    0.0 \\ \hline
    $s_1$ & $a_0$      & 0.0   &     0.0 &     1.0 &    0.0 \\ \hline
    $s_2$ & $a_0$      & 0.0   &     0.0 &     0.0 &    1.0 \\ \hline 
    $s_3$ & $a_0$      & 1.0   &     0.0 &     0.0 &    0.0 \\ \hline 
    $s_0$ & $a_1$      & 0.0   &     0.0 &     0.0 &    1.0 \\ \hline
    $s_1$ & $a_1$      & 1.0   &     0.0 &     0.0 &    0.0 \\ \hline
    $s_2$ & $a_1$      & 0.0   &     1.0 &     0.0 &    0.0 \\ \hline 
    $s_3$ & $a_1$      & 0.0   &     0.0 &     1.0 &    0.0 \\ \hline 
    \end{tabular}
\end{center}

From this we can infer that action $a_0$ causes the state number to increment, and $a_1$ causes the state number to decrement.  It is deterministic because the specified action always works 100\% of the time.  For a given state and action (a row of the table), the probabilities sum to 1.0.
    
The corresponding \textcolor{accepted}{MDP states} graph would be:


\begin{figure}[h]
	\centering
	\fbox{
		\includegraphics[width=3in]{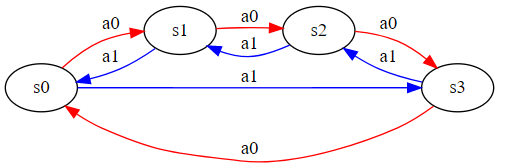}
	}
	\caption{MDP states graph for a small, deterministic MDP}
\end{figure}

\newpage
\section{Appendix 2: Detailed Pseudocode}

\begin{algorithm}[H]
\caption{\algo}\label{algorithm}
\begin{algorithmic}[1]
\Procedure{\algo}{$\textit{rewardSources}$}
\State $\textit{valueFunction} \gets \text{preallocate array of zeros} $ 
\State $\textit{processedPeaks} \gets \text{empty list}$
\State $\textit{sortedPeaks} \gets \textit{PrecomputePeaks( rewardSources )}$
\While {$\textit{sortedPeaks} \text{ is not empty}$}
	\State $\textit{deltaPeaks} \gets \textit{ComputeDeltas( valueFunction )}$
	\State $\textit{sortedPeaks} \gets \textit{PruneInvalidPeaks()}$
	\State $\textit{maxPeak} \gets \textit{max( [ sortedPeaks, deltaPeaks ] )}$
	\State $\textit{sortedPeaks} \gets \textit{RemoveAffectedPeaks( maxPeak )}$
	\State $\textit{valueFunction} \gets \textit{UpdateValueFunction()}$
\EndWhile
\Return $\textit{valueFunction}$
\EndProcedure
\end{algorithmic}
\end{algorithm}

Line 2 initializes memory to hold the value function.  Line 3 initializes an empty list to track which peaks have been processed by the algorithm.  Line 4 pre-computes baseline peaks and combined peaks based off a list of reward sources and stores them in the form of a sorted list, sorted by value of each peak.  Lines 5-10 continue until we have exhausted the potential peaks and each iteration of the loop whittles away at the list of possible peaks.  Line 6 computes delta peaks for any remaining reward sources utilizing the value function that has been computed so far.  Line 7 removes any peaks that have become invalid.  Line 8 selects the peak with maximum value.  Line 9 removes any other potential peaks in the list that are affected by selecting the peak with maximum value.  (For example, a combined peak with states 3 and 4 are selected.  The baselines for states 3 and 4 would then be removed.)  Line 10 then updates the value function based off the newly selected peak with maximum value.

\begin{algorithm}[H]
\begin{algorithmic}[1]
\Procedure{PrecomputePeaks}{rewardSources}
\State $\textit{list} \gets \text{empty SortedList} $ 

\ForAll{rewardSources}
	\State $\textit{list.add}$( baseline peak for reward source )
\EndFor

\ForAll{rewardSources}
	\State $nbr \gets \text{find neighboring state with highest reward}$
	\If {nbr is not empty}
	   \State $\textit{list.add}$( cycle peak for reward source )
    \EndIf
\EndFor
\Return $\textit{list}$
\EndProcedure
\end{algorithmic}
\end{algorithm}

Line 2 initializes a sorted list that is sorted by value of the peaks.  In Lines 3-4, a baseline peak is computed for each reward sourcea.  In lines 5-8, if any reward sources are next to each other, their combined peaks are computed.

$\textit{PrecomputePeaks()}$ is a $O(|R| \times |A|)$ algorithm that is done one time at the beginning of the algorithm and yields a list with worst case length of $O(|R| \times |A|)$ entries (but only if the reward sources are all adjacent to each other).  

\begin{algorithm}[H]
\begin{algorithmic}[1]
\Procedure{ComputeDeltas}{}
\State $\textit{list} \gets \text{empty SortedList} $ 

\ForAll{reward sources}
	\State $\text{compute delta of reward and value function}$
	\State $nbr \gets \text{find neighboring state with highest value}$
    
	\State $\textit{list.add}( max( delta peak, neighbor value ) )$
\EndFor
\EndProcedure
\end{algorithmic}
\end{algorithm}

Line 2 initializes a sorted list that is sorted by value of the peaks.  Lines 3-6 compute delta peak for any reward sources that remain.  Line 5-6 properly sort the delta with respect to neighboring states.

$\textit{ComputeDeltas( valueFunction )}$ is a $O(|R|\times |A| )$ algorithm that is done for each pass of the loop.

\begin{algorithm}[H]
\begin{algorithmic}[1]
\Procedure{PruneInvalidPeaks}{}
\ForAll{remaining peaks}
	\State $nbr \gets \text{find neighboring state with highest value}$
	\If {$nbr > peak$}
	    \State $\textit{list.remove( peak )}$
    \EndIf
\EndFor
\EndProcedure
\end{algorithmic}
\end{algorithm}

Lines 2-5 remove any peaks that have become invalid.

$\textit{PruneInvalidPeaks()}$ is a $O(|R| \times |A| )$ algorithm that is done for each pass of the loop, but this also shrinks by $O(|A|)$ entries each pass.

\begin{algorithm}[H]
\begin{algorithmic}[1]
\Procedure{RemoveAffectedPeaks}{list, state}
\ForAll{remaining peaks}
	\If {peak is affected by state}
	    \State $\textit{list.remove( peak )}$
    \EndIf
\EndFor
\EndProcedure
\end{algorithmic}
\end{algorithm}

Lines 2-4 remove any peaks that have been eliminated by the choice of the peak with maximum value.

$\textit{RemoveAffectedPeaks}$ operates over the $O(|R| \times |A|)$ $\textit{sortedPeaks}$ list, but this also shrinks by $O(|A|)$ entries each pass.

\begin{algorithm}[H]
\begin{algorithmic}[1]
\Procedure{UpdateValueFunction}{value\_function, peak}
	\State $\textit{interim} \gets \textit{Propagate(peak)}$
	\State $\textit{valueFunction} \gets \textit{element-wise-max( interim, valueFunction)}$
\EndProcedure
\end{algorithmic}
\end{algorithm}

Line 2 propagates the peak outward to compute an interim value function for the reward source selected during this iteration.  Line 3 then performs an element-wise max operation over the value function computed during the previous iteration and the interim value function, resulting in the value function for this iteration.

$\textit{UpdateValueFunction}$ is a $O(|S|)$ operation due to the call to $\textit{Propagate}$.

\begin{algorithm}[H]
\begin{algorithmic}[1]
\Procedure{Propagate}{peak}
	\State $\textit{valueFunc} \gets \text{allocate empty array of zeros}$
    \ForAll{states}
       \State $\textit{valueFunct[state]} \gets \textit{peak} \times \textit{discount}^\textit{ConnDist} $
    \EndFor
    \Return $\textit{valueFunc}$
\EndProcedure
\end{algorithmic}
\end{algorithm}

Line 2 creates a value function with all zeros.  Lines 3-4 compute the value function based off the peak's value, the distance through the transition graph, and the discount factor.

$\textit{Propagate()}$ is a $O(|S|)$ operation.  

\begin{algorithm}[H]
\begin{algorithmic}[1]
\Procedure{ConnDist}{}
    \State $\textit{dist} \gets \text{manhattan distance between start and end state}$\\
    \Return $\textit{dist}$
\EndProcedure
\end{algorithmic}
\end{algorithm}

Line 2 calculates the distance through the transition graph.  Note that because our 2D grid world has a known structure, we can take advantage of this knowledge to perform our distance calculation in constant time rather than having to perform a shortest path search through the graph.  This algorithm will receive an important performance boost whenever this is possible.  (The overhead of performing a graph search through the transition graph for a general MDP may outweigh the benefits of this algorithm.  That is an open question for future work.)

$\textit{ConnDist()}$ for this 2D grid world is a $O(1)$ constant operation.  
\subsubsection{Time Complexity}

Because this algorithm still requires the full value function to be computed, this drives an underlying $O(|S|)$ time complexity for creating the data structure and updating the value function at the end of each pass.

In general the time complexity of this algorithm is $O( |R|^2 \times |A|^2  \times |S|)$, where $|R|$ is the number of reward sources, $|A|$ is the number of actions, and $|S|$ is the number of states.

For environments where the connected distance is not easily determined (arbitrary transition graph), then the complexity to determine the distance between states must be taken into consideration.  However, it is assumed that this can be precomputed offline because $T$ is assumed to be stationary. 

For environments like the 2D grid world where the structure of the space is known, determining the connected distance between states is a $O(1)$ calculation.

\subsubsection{Memory Complexity}

Memory complexity for the algorithm is $O( |S| + |R| \times |A|)$

\newpage
\section{Appendix 3: Single Source Proof}

\setcounter{theorem}{0}

\subsection{Proof for single sources}

\begin{definition}
Single reward source:

We define a MDP with a single reward source as having a state $s_g$ where a positive real reward $r_g > 0$ is obtained.  That is, with a reward function of:

\begin{equation}
R(s, a) = \begin{cases}
                    r_g>0 & \mbox{if } s=s_g \\
    				0  & \mbox{otherwise}.
			  \end{cases}
              \label{eq:reward_cont}
\end{equation}
\end{definition}

\begin{theorem}
Every deterministic continuous fully connected MDP with a single reward source has a minimum cycle.
\label{thm:min_cycle_exists}
\end{theorem}

\begin{proof}
Assume that we reach the goal state $s_g$ whereupon we obtain reward $r_g$.  We must then choose an action $a \in \A$ which will select our next state.  We know from the definition of the reward function in Equation \ref{eq:reward_cont} that immediate reward is 0 in all states other than $s_g$; therefore, the only way to accumulate any new reward is to take a set of actions that result in a return to state $s_g$, which we termed a cycle which we can denote here as $D_c$.  We observe that since reward can only be collected at $s_g$, that the reward per cycle that is collected is $R_c = \gamma ^ {D_c} \times r_g$.  We observe that $R_c$ grows as $D_c$ decreases, with the max occurring at $D_{c_{max}} = 1$.  Thus, we prove that a cycle must exist for a single reward source, and that cycles with shorter length are preferred.

We must consider two types of transition functions, $T$:
\begin{enumerate}
\item[Case 1:]  Those that allow self transitions ($s_g \rightarrow s_g$ takes one action).

For MDPs which have a transition matrix $T$ that allow for an action to stay in the same state, it is possible for our action $a$ above to transition from the assumed start state of $s_g$ to a next state of $s' = s_g$.  (This transition is not possible in the 2D grid world we use for illustration, but is possible for a general deterministic continuous MDP, so we include this case for consideration.)  In this case, we say that the minimum cycle distance $\phi(s_g) = 1$ because it takes one action to go from $s_g$ to itself.

\item[Case 2:]  Those that do not allow self transitions ($s_g \rightarrow s_g$ cannot be accomplished in one action, but instead leads to some next state $s'$ which is distinct from $s_g$).

From $s'$ we can obtain the distance back to the goal state $s_g$ with $\delta(s', s_g)$.  We now consider the following possible cases of $\delta(s', s_g)$, which is 1, or a finite $k$.  Note that our assumption of a fully connected MDP by definition means that all states are connected, meaning for all $s'$, $\delta(s', s_g) < \infty$.

\begin{enumerate}
\item[Case 2.a:] $\delta(s', s_g) = 1$:

For $\delta(s', s_g) = 1$, we can then conclude that to return to the goal state $s_g$, we would simply take action $a^* \in \A | f(s', a^*) = s_g$, thus establishing a minimum cycle with distance $\phi(s_g) = 2$ from $s_g$ back to itself.  

\item[Case 2.b:] $\delta(s', s_g) = k$, where $1 < k < \infty$:

For $\delta(s', s_g) = k$ where $1 < k < \infty$, we can similarly conclude that to return to the goal state $s_g$, we would simply take a sequence of actions $a^{(i)} \in \A | f(s', a^{(1)}, a^{(2)}, ... a^{(k)}) = s_g$, thus establishing a minimum cycle distance $\phi(s_g) = k + 1$ from $s_g$ back to itself.  
\end{enumerate}
\end{enumerate}

Thus we have established that for a continuous deterministic fully connected MDP, a minimum cycle must exist.
\end{proof}

\begin{theorem}
For a deterministic continuous fully connected MDP with a single reward source $r_g$ at state $s_g$, the value at $s_g$ is equal to:
\label{thm:single_source_peak_value}
\end{theorem}
\begin{equation}
\begin{split}
  V(s_g) &= \frac{r_g}{1 - \gamma^{\phi(s_g)}},
\end{split}
\end{equation}

where $\phi(s_g)$ is the minimum cycle distance for the MDP.

\begin{proof}
From Theorem \ref{thm:single_source_peak_value}, we know that a minimum cycle must exist for a deterministic continuous fully connected MDP.  To determine the value at $s_g$, we again consider taking an action $a \in \A$ from $s_g$ which leads to state $s'$, where the distance from $s'$ back to $s_g$ is 0, 1, or a finite $k$.  Note again that our assumption of a fully connected MDP by definition means that all states are connected, meaning for all $s'$, $\delta(s', s_g) < \infty$.

\begin{enumerate}
\item[Case 1:] $\delta(s', s_g) = 0$:

For MDPs which have a transition matrix $T$ that allow for an action to stay in the same state, it is possible for our action $a$ above to transition from the assumed start state of $s_g$ to a next state of $s' = s_g$.  (This transition is not possible in the 2D grid world we use for illustration, but is possible for a general deterministic continuous MDP, so we again include this case for consideration.)

Starting from $s_g$ and taking action $a^* \in \A^*$ as defined in case 1 of Theorem \ref{thm:single_source_peak_value}, we obtain an immediate reward of $r_g$ and find that our next state $s' = s_g$.  We then again take action $a^*$ and receive immediate reward $r_g$, making the cumulative reward $R = r_g + \gamma \times r_g$.  As we continue to take take action $a^*$, we find that in the limit the cumulative reward is $R = r_g + \gamma \times r_g + \gamma^2 \times r_g + \gamma^3 \times r_g ...$.  Note that as $0 < \gamma < 1.0$, this is a convergent geometric series with a limit of $\frac{r_g}{1-\gamma}$.  As in Theorem \ref{thm:single_source_peak_value} we have shown that the minimum cycle distance for this case $\phi(s_g) = 1$, the value at $s_g$ can be expressed equivalently as:

\begin{equation}
\begin{split}
  V(s_g) &= \frac{r_g}{1 - \gamma^{\phi(s_g)}}
\end{split}
\end{equation}

\item[Case 2:] $\delta(s', s_g) = 1$:

For $\delta(s', s_g) = 1$, we know by definition that at least one action $a^* \in \A | f(s', a^*) = s_g$ exists that will lead back to $s_g$, and we define all other actions $a^- = \A \setminus a^*$.  We know from our reward function that the only state in which reward is non-zero is $s_g$, thus taking an action $a^*$ will result in reward $r_g$ and taking an action $a^-$ will result in no reward, thus action $a^*$ is optimal.  We may also concluded that taking $a^*$ will result in a minimum cycle distance of $\phi(s_g) = 2$, yielding total reward of $R = r_g + \gamma \times \gamma \times r_g = r_g + \gamma^2 r_g$.  If we repeat this procedure, we then obtain reward two steps later yielding total reward of $r_g + \gamma^2 r_g + \gamma^4 r_g$.  In the limit, the cumulative reward (in other words, the value) is a geometric series:  

\begin{equation}
\begin{split}
  V(s_g) &= r_g + \gamma^2 r_g + \gamma^4 r_g + ... \\
       &= \frac{r_g}{1 - \gamma^{2}} \\
       &= \frac{r_g}{1 - \gamma^{\phi(s_g)}}
\end{split}
\end{equation} 

\item[Case 3:] $\delta(s', s_g) = k$, where $1 < k < \infty$:

For $\delta(s', s_g) = k$ where $1 < k < \infty$, we can similarly conclude that to return to the goal state $s_g$, we would simply take a sequence of actions $a^{(i)} \in \A | f(s', a^{(1)}, a^{(2)}, ... a^{(k)}) = s_g$, yielding total reward of $r_g + \gamma^{k+1} r_g$ and a minimum cycle distance $\phi(s_g) = k + 1$.  If we repeat this procedure, we then obtain reward $k + 1$ steps later yielding total reward of $r_g + \gamma^{k+1} r_g + \gamma^{2(k+1)} r_g$, and so on.  In the limit, the reward is:  

\begin{equation}
\begin{split}
  V(s_g) &= r_g + \gamma^{k+1} r_g + \gamma^{2(k+1)} r_g + \gamma^{4(k+1)} r_g + ... \\
       &= \frac{r_g}{1 - \gamma^{k+1}} \\
       &= \frac{r_g}{1 - \gamma^{\phi(s_g)}}
\end{split}
\label{eq:v_sg_proof}
\end{equation}

\end{enumerate}

Thus we have established the value at the state $s_g$ where reward $r_g$ is collected.
\end{proof}

\begin{theorem}
The value function for a deterministic continuous fully connected MDP with a single reward source with discount factor $0 < \gamma <1$ has the form: 

\begin{equation}
\begin{split}
V(s) &= \gamma^{\delta(s,s_g)} \times V(s_g)
\label{eq:cont_cycle}
\end{split}
\end{equation}
where $s \in \mathcal{S}$, and where $V(s_g) = \frac{r_g}{1 - \gamma^{\phi(s_g)}}$ and $\phi(s_g)$ is the minimum cycle distance for the MDP, as established in Theorem \ref{thm:single_source_peak_value}.
\label{thm:val_function}
\end{theorem}

\begin{proof}
Given then, that we now know that a minimum cycle exists, and that we know the value at the state $s_g$ where reward $r_g$ is collected, we turn to examine the value at all other states.  Given our definition of the reward function for a single source MDP, we note here that in all states other than $s_g$, no reward is collected.  We will prove by induction.

Let us now assume that we start not at state $s_g$, but at some state one action away from $s_g$, which we will denote as $s^{(1)} | \delta(s^{(1)}, s_g) = 1$.  Note that we already know from Theorem \ref{thm:single_source_peak_value} the value we will obtain once we reach state $s_g$, which we refer to here as $V(s_g)$.  We therefore know that since we must take one step to obtain this value $V(s_g)$, then the future discounted reward is then $\gamma \times V(s_g)$.  As no immediate reward is collected at state $s^{(1)}$ we know that the expected value at state $s^{(1)}$ is then simply the future discounted reward:  $V(s^{(1)}) = \gamma \times V(s_g)$.

\begin{equation}
\begin{split}
  V(s^{(1)}) &= \gamma \times V(s_g) \\
             &= \gamma ^ {\delta( s^{(1)}, s_g )} \times V(s_g) 
\end{split}
\end{equation} 

We now consider the the case where we have a minimum cycle distance of $s^{(n)} | \delta(s^{(n)}, s_g) = n$ and $s^{(n+1)} | \delta(s^{(n+1)}, s_g) = n+1$.  We can see clearly from the definition of the cycle distance that for any state $s^{(n+1)}$ there exists an action $a^* \in \A | f( s^{(n+1)}, a^* ) = s^{(n)}$.  Also given that reward is only collected at $s_g$, we again have no immediate reward when transitioning from state $s^{(n+1)}$ to state $s^{(n)}$ and need only consider future discounted reward.  Thus:

\begin{equation}
\begin{split}
  V(s^{(n+1)}) &= \gamma \times V(s^{(n)})  \\
\end{split}
\end{equation} 

This means that:

\begin{equation}
\begin{split}
  V(s^{(2)}) &= \gamma \times V(s^{(1)})  \\
             &= \gamma \times \gamma ^ {\delta( s^{(1)}, s_g )} \times V(s_g)  \\
             &= \gamma \times \gamma \times V(s_g) \\
             &= \gamma ^ {\delta( s^{(2)}, s_g )} \times V(s_g) 
\end{split}
\end{equation} 

Then by induction, we see that the value of any state $s$ is as follows, completing the proof:
\begin{equation}
\begin{split}
  V(s) &= \gamma^{\delta(s,s_g)} \times V(s_g)
\end{split}
\end{equation} 
\end{proof}

\newpage
\section{Appendix 4: Proof of Algorithm}

We remind the reader that the proofs for the algorithm in this section are claimed to only apply to a narrow subset of MDPs:
\begin{enumerate}
\item Deterministic continuous MDPs
\item Reward function based only on state (not action)
\item Only positive, real rewards (no negative rewards)
\end{enumerate}

While we will explore in future papers whether this method can be applied to a larger class of MDPs, we will start with this narrow definition.
To prove that our algorithm results in an optimal value function, we must prove the resulting value function satisfies Bellman optimality equation $V^* = L V^*$.  This is a complex, multi-step proof.  In Part 1, we establish that optimal value function is reached when the maximum possible value is found at each state.  In Part 2, we show that our algorithm's effect is to determine the maximum possible value at each state, thereby satisfying the Bellman optimality equation.

\subsection{Part 1: Bellman optimality and maximum value}

The Bellman equation and Bellman operator have been well studied by many sources.  For completeness, we repeat proofs available elsewhere regarding monotonicity, contraction over the max norm, and the uniqueness of the fixed point solution $V^*$.

\textbf{Monotonicity:} First, we repeat the well known property that the Bellman operator $L$ satisfies the property of monotonicity, which means that for any two value functions $V$ and $V'$ and given $V \leq V'$, then $L V \leq L V'$.

\begin{proof}

Note that the $\leq$ operator here is an element-wise operator:
$V \leq V' \rightarrow \forall s, V(s) \leq V'(s)$.

We then translate the inequality to a more convenient equivalent form:
\begin{align*}
L V &\leq L V'  \\
L V - L V' &\leq 0 \\
L[V(s)] - L[V'(s)] &\leq 0, \forall s
\end{align*}

Then, for all $s \in S$:
\begin{align*}
L[V(s)] &- L[V'(s)] &\leq 0 \\
R(s) + \max_a \big[ \gamma \sum_{s'} T( s, a, s') V(s') \big] &- R(s) - \max_a \big[ \gamma \sum_{s'} T( s, a, s') V'(s') \big] &\leq 0 \\
\max_a \big[ \gamma \sum_{s'} T( s, a, s') V(s') \big] &- \max_a \big[ \gamma \sum_{s'} T( s, a, s') V'(s') \big] &\leq 0 \\
\end{align*}

Since we know that $V \leq V'$, we then know that:
\begin{equation*}
\begin{split}
\max_a \big[ \gamma \sum_{s'} T( s, a, s') V(s') \big] - \max_a \big[ \gamma \sum_{s'} T( s, a, s') V'(s') \big]  \leq \\ \max_a \big[ \gamma \sum_{s'} T( s, a, s') V(s') \big] - \max_a \big[ \gamma \sum_{s'} T( s, a, s') V(s') \big] \\
\max_a \big[ \gamma \sum_{s'} T( s, a, s') V(s') \big] - \max_a \big[ \gamma \sum_{s'} T( s, a, s') V'(s') \big]  \leq 0
\end{split}
\end{equation*}
\end{proof}

\textbf{Contraction mapping:}  We then recall that the Bellman operator is a contraction over the max norm $|\cdot|_\infty$.


\begin{proof}
A contraction operator means that for any two functions $f$ and $g$,

\begin{equation*}
|\max_a f(a) - \max_a g(a)| \leq \max_a | f(a) - g(a) |,
\end{equation*}
where, again, the $\leq$ operator is taken to be element-wise and is true for all $a$.

Assume that $\max_a f(a) \geq \max_a g(a)$, that $a^* = arg max_a f(a)$ so that $\max_a f(a) = f(a^*)$.  Then,

\begin{align*}
|\max_a f(a) - \max_a g(a)| &= f(a^*) - \max_a g(a) \\
\end{align*}

Given our assumption that $\max_a f(a) \geq \max_a g(a)$, then $f(a^*) \geq g(a^*)$ and $f(a^*) - g(a^*)$ is positive.  Then:

\begin{align*}
|\max_a f(a) - \max_a g(a)| &\leq f(a^*) - g(a^*) \\
\end{align*}

Then from the definition of absolute value (also given that $f(a^*) - g(a^*)$ is a positive value):

\begin{align*}
|\max_a f(a) - \max_a g(a)| &= |f(a^*) - g(a^*)| \\
                            &= \max_a |f(a) - g(a)|
\end{align*}

Then, to prove that the Bellman operator is a contraction mapping, we must prove that:

\begin{align*}
|| L V - L V' ||_\infty \leq \gamma || V - V' ||_\infty
\end{align*}

From the definition of the max norm, for all $s \in S$:
\begin{align*}
| L V(s) - L V'(s) | \leq \gamma || V - V' ||_\infty
\end{align*}

\begin{align*}
| L V(s) - L V'(s) | &= |R(s) + \max_a \gamma \sum_{s'} T(s, a, s') V(s') - R(s) - \max_a \gamma \sum_{s'} T(s, a, s') V'(s') | \\
                     &= |\max_a \gamma \sum_{s'} T(s, a, s') V(s') - \max_a \gamma \sum_{s'} T(s, a, s') V'(s') | \\
\end{align*}

From our contraction mapping proof above, we can then use the result to say:
\begin{align*}
| L V(s) - L V'(s) | &= \max_a | \gamma \sum_{s'} T(s, a, s') V(s') -  \gamma \sum_{s'} T(s, a, s') V'(s') | \\
                     &= \max_a \gamma | \sum_{s'} T(s, a, s') ( V(s') - V'(s') ) | \\
                     &= \max_a \gamma \sum_{s'} T(s, a, s') | V(s') - V'(s') | \\
                     &\leq \max_a \gamma \sum_{s'} T(s, a, s') | V(s') - V'(s') | \\
                     &\leq \max_a \gamma \sum_{s'} T(s, a, s') || V - V' ||_\infty \\
\end{align*}

Given that we know $\sum_{s'} T(s, a, s') = 1$ for a given $a$:

\begin{align*}
| L V(s) - L V'(s) |  &\leq \max_a \gamma || V - V' ||_\infty \\
                      &\leq \gamma || V - V' ||_\infty \\
\end{align*}
\end{proof}

\textbf{Stationary:}  Then, given that it is a contraction mapping over the max norm, we repeat the well known property that the Bellman operator has a unique solution $V^*$ and that this optimal solution is fixed (or stationary under L).

\begin{proof}
We prove by counter proof.  Assume that there are two value functions $V, V'$ that are both fixed points under L, $V = L V$ and $V' = L V'$.

However, from our contraction mapping proof, we know that $||L V - L V' ||_\infty \leq \gamma || V - V' ||_\infty$.  But with our assumption that $LV = V$ and $LV' = V'$, this becomes:

\begin{align*}
||L V - L V' ||_\infty = || V - V' ||_\infty
\end{align*}

Recalling that $0 < \gamma < 1$, then we have the contradiction that:
\begin{align*}
|| V - V' ||_\infty \leq \gamma || V - V' ||_\infty
\end{align*}

which could only be true if $V$ and $V'$ are all zeros, or if $V = V'$, which in both cases reduces to $V = V'$, proving that there must only be one unique fixed point solution for the Bellman operator $L$.
\end{proof}

\begin{theorem}
Given the optimal policy $\pi^*$ and the associated value function $V^*$, the optimal value function has maximum value at each state:
\begin{equation*}
V^* \geq V^\pi, \forall \pi
\end{equation*}
\label{thm:opt_value_func}
\end{theorem}

\begin{proof}
Given the monotonicity of the Bellman operator $L$, we know that at each step of value iteration $L V \leq L V'$.  Given that the Bellman operator is a contraction, we also know that successive applications of the Bellman operator converge to the optimal policy $\pi^*$ with a corresponding value function $V^*$.  And because we know that the optimal value function $V^*$ is a unique, fixed point solution of $V^* = L V^*$, we know that once we reach the optimal solution under the contraction we will never diverge from it.  Thus the sequence of value functions is $V_0 \leq V_1 \leq V_2 \cdots \leq V^* \leq V^* \leq V^* ...$ and we can then conclude that $\forall s, \pi: V^*(s) \geq V^\pi(s)$. 
\end{proof}

Thus, to prove that the algorithm satisfies the Bellman optimality equation, we must show that the algorithm determines the maximum possible value at each state $s$.

\subsection{Part 2: Algorithm calculation of max value}

We turn now to examine the way in which reward is collected and how we can determine whether the algorithm in fact calculates the maximum value at every state.

\subsubsection{Zero versus Positive Reward}

We start first with the observation that the states can be broken into two general categories:  those with reward, which we define as $S^+ = s \in S | R(s) > 0$,  and those without reward, which we define as $S^Z = s \in S | R(s) = 0$.  (Recall that our definition of the reward function permits only positive rewards and that the rewards are based on the state and not on the action.  At this time, we do not claim to have solved the problem of rewards based on the action.)  Note that $S^Z$ and $S^+$ may be a collection of disjoint subsets of $S$.
\newcommand{\DrawSplus}[4]{
 \tkzDefPoint(#1+#3*(0.4   ) ,#2+#4*0.0   ){A}
 \tkzDefPoint(#1+#3*(1     ) ,#2+#4*0.4   ){B}
 \tkzDefPoint(#1+#3*(0.5   ) ,#2+#4*1.0   ){C}
 \tkzDefPoint(#1+#3*(0.1773) ,#2+#4*0.4421){D}
 \tkzDefPoint(#1+#3*(0.2   ) ,#2+#4*0.2   ){M}
 
 
 \tkzLabelPoint[above right](M){$S^+$}

 \draw [thick] plot [smooth cycle, tension=0.7] coordinates {
 (A) (B) (C) (D) 
 };
}
\newcommand{\DrawSzero}[4]{
 \tkzDefPoint(#1+#3*(0.2+(random*0.1)   ) ,#2+(#4*(0.1   +(random*.1) ) )  ){A}
 \tkzDefPoint(#1+#3*(0.5+(random*0.1)   ) ,#2+(#4*(0.18  +(random*.1) ) )  ){B}
 \tkzDefPoint(#1+#3*(0.7+(random*0.1)   ) ,#2+(#4*(0.15  +(random*.1) ) )  ){C}
 \tkzDefPoint(#1+#3*(1.0+(random*0.1)   ) ,#2+(#4*(0.3421+(random*.1) ) )  ){D}
 \tkzDefPoint(#1+#3*(0.8+(random*0.1)   ) ,#2+(#4*(0.9   +(random*.1) ) )  ){E}
 \tkzDefPoint(#1+#3*(0.3+(random*0.1)   ) ,#2+(#4*(0.7   +(random*.1) ) )  ){F}
 \tkzDefPoint(#1+#3*(0.2+(random*0.1)   ) ,#2+(#4*(0.4   +(random*.1) ) )  ){G}
 \tkzDefPoint(#1+#3*(0.4+(random*0.1)   ) ,#2+(#4*(0.4   +(random*.1) ) )  ){M}
 
 
 \tkzLabelPoint[above right](M){$S^Z$}

 \draw [thick] plot [smooth cycle, tension=0.7] coordinates {
 (A) (B) (C) (D) (E) (F) (G)
 };
}

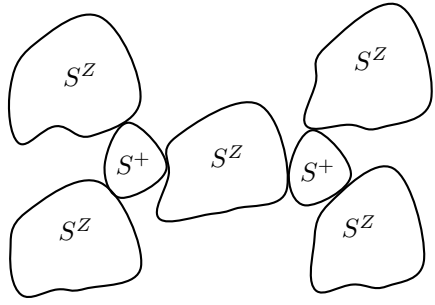
\begin{figure}[H]
\centering
\begin{tikzpicture}[label]
 \DrawSzero{0}{1}{2}{2}
 \DrawSzero{3.9}{1.05}{2}{2}
 \DrawSzero{3.9}{3.19}{2}{2}
 \DrawSzero{0}{3.05}{2}{2}
 \DrawSzero{2}{2}{2}{2}
 \DrawSplus{3.95}{2.5}{1}{1}
 \DrawSplus{1.5}{2.6}{1}{1}
\end{tikzpicture}
\caption{Disjoint subsets of $S^Z$ and $S^+$ that form $S$.}
\end{figure}
               
\begin{theorem}
The maximum of the value function cannot occur in states where the reward is 0.
\end{theorem}

\begin{proof}
If we examine the recursive form of the Bellman equation at the optimal policy $\pi^*$ with the (stationary) $V^*$:

\begin{equation}
V^*(s) = R(s) + \gamma \max_a \sum_{s'} T(s, a, s') V^*(s')
\end{equation}

then for $s_z \in S^Z$ the value $V^*(s_z)$ is the discounted future reward.  Thus, if $a^* = arg max_a \sum_{s'} T(s_z, a, s') V^*(s')$, then $V^*(s_z) = \gamma \sum_{s'} T(s_z, a^*, s') V^*(s')$.  And given that the discount factor $0 < \gamma < 1$, we see that  $V^*(s_z) \leq \sum_{s'} T(s_z, a^*, s') V^*(s')$.  Furthermore, if $V^*(s') > 0$, then $V^*(s_z) < \sum_{s'} T(s_z, a^*, s') V^*(s')$, which is to say that $V^*(s_z)$ can only be equal to $V^*(s')$ if both are zero.  (We do not need to prove it here, but we will state that this can only occur if the value function is zero everywhere.)

Thus, as we know that $V^*(s_z)$ is strictly less than $V^*(s')$ we know that a maximum of the value function cannot occur for $s_z \in S^Z | R(s_z) = 0$.
\end{proof}

The converse then is that if the maximum of the value function $V^*$ cannot occur where $R(s) = 0$ (that is in $S^Z$), then it must occur where $R(s) > 0$ (that is, in $S^+$).

\begin{theorem}
States with reward of zero, $S^Z$, are determined from the states with non-zero reward, $S^+$.
\label{thm:determinism}
\end{theorem}

\begin{proof}
From the relation $V^*(s_z) = \gamma \sum_{s'} T(s_z, a^*, s') V^*(s')$ that was developed in the previous proof, we can conclude that for all $s_z \in S^Z$, the resulting value function is determined solely by the value at another state, through the discounted future reward.  Thus, to know the value for any state $s_z \in S^Z$, we must look to another state to define the value.

Consider a chain of states in $S^Z$, \{ $s_z^{(1)}, s_z^{(2)}, \cdots, s_z^{(n)}$ \} and suppose that each element in the chain is the result of the optimal action at each step that satisfies $a* = arg max_a \gamma \sum_{s'} T(s_z, a, s') V^*(s')$.  What can we say of the value of these states?  We can say nothing, as none of them have any immediate value $R(s) > 0$.  Let us say that at the next optimal action $a^*$ we reach a state in $S^+$.  At this point in time we can definitively say that $V(s_z^{(1)}) > 0$ as the state in $S^+$ has an immediate reward greater than 0, and thus through the discount factor all states in our chain obtain some positive value.

Let us repeat this experiment for all states in $S^Z$.  In general, starting from any state $s_z \in S^Z$ and taking the optimal action $a^*$ at each step, we will form a chain of $s_z^{(1)}, s_z^{(2)}, \cdots, s_z^{(n)}, s_p$ with length $n+1$ that will terminate in a state $s_p \in S^+$.  At each step in the chain due to zero immediate reward in $S^Z$, the value $V( s_z^{(k)} ) = \gamma \times V( s_z^{(k+1)} )$, where $k = \{ 1 ... (n-1) \}$.  And finally when the chain terminates at $s_p \in S^+$, $V( s_z^{(n)} ) = \gamma * V( s_p )$.  Thus by induction we have shown that all states in $S^Z$ have a value that is determined by a state in $S^+$.
\end{proof}

Illustrating this concept:

\begin{figure}[H]
\centering
\begin{tikzpicture}[label]
 \DrawSzero{0}{1}{8}{5}
 \DrawSplus{8.3}{2.5}{1}{1}
 \tkzDefPoint(5.5,4  ){A}
 \tkzDefPoint(5.8,3.5){B}
 \tkzDefPoint(6.0,3.2){C}
 \tkzDefPoint(6.5,2.9){D}
 \tkzDefPoint(7.2,3.0){E}
 \tkzDefPoint(8.5,3.0){F}
 \tkzSetUpPoint[fill=red]
 \tkzDrawPoints[color=red, size=10](A,B,C,D,E,F)
 \draw [thick, densely dotted] plot [color=red] coordinates {
 (A) (B) (C) (D) (E) (F) 
 };
 \tkzLabelPoint[above right](A){$s_z^{(1)}$}
 \tkzLabelPoint[above right](B){$s_z^{(2)}$}
 \tkzLabelPoint[below      ](E){$s_z^{(n)}$}
 \tkzLabelPoint[below left=3pt ](F){$s_p$}
 
\end{tikzpicture}
\caption{Optimal sequence of actions through $S^Z$ until a point in $S^+$ is reached.}
\end{figure}
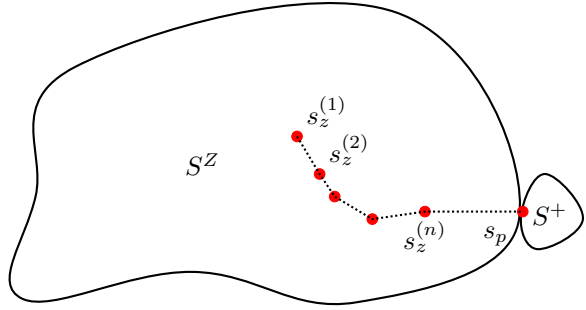

Thus, we have established that all states where reward is zero are deterministic with respect to states with positive reward, and that the maximum of the value function cannot occur where reward is zero.  Thus, in order to fully determine the value function, we need only consider the states where reward occur.  This is a key conclusion that underpins the algorithm.

We now expand on the nature of the maximum value at each state of the optimal value function.

\begin{theorem}
Given the optimal policy $\pi^*$ and the associated value function $V^*$, the optimal value function at each state is equivalent to:
\begin{equation*}
V^*(s) = \max_\pi V^\pi(s), \forall s \in S
\end{equation*}
\label{thm:opt_value_func}
\end{theorem}

\begin{proof}
\begin{equation*}
\begin{split}
V^* &\geq V^\pi, \forall \pi \\
V^*(s) &\geq V^\pi(s), \forall \pi, \forall s \in S \\
       &\geq \max_\pi V^\pi(s), \forall s \in S \\
\end{split}
\end{equation*}

As we know by definition that $V^* \in V^\pi$ and we have already proven that $V^*$ is in fact the maximum value at each state, we can strengthen our statement with:

\begin{equation*}
\begin{split}
V^*(s) &= \max_\pi V^\pi(s), \forall s \in S \\
\end{split}
\end{equation*}
\end{proof}

Thus we need only prove that the algorithm produces the maximum value at each state in $S^+$.

\subsubsection{Reward Collection}

We now examine the possible ways that each reward can possibly be collected:

\begin{definition}
Given a policy $\pi$, select an initial state $s^{(0)}$ and follow the policy.  The resulting path through the state space is $p = \{ s^{(i)}, ... \}, \forall i = \{0 ... \infty\}$.  If a given state $s_k = s^{(i)}$ for some $i$, then we say that $s_k$ has been visited.  If a reward $r_k$ is present at $s_k$, then we additionally say that reward $r_k$ has been collected.  We denote the count of the number of times that $r_k$ has been collected as $N_k$.
\end{definition}

\begin{theorem}
All rewards $R = \{ r_1, r_2, ..., r_N \}$ are collected either once or infinitely under a given policy $\pi$.  That is, for a given reward $r_i \in R, N_k = \{ 1, \infty \}$ , and only rewards falling within a minimum cycle of a local maximum in the value function are collected infinitely. 
\end{theorem}

\begin{proof}
To prove that all rewards in an MDP are collected at least once, we note that a \textcolor{accepted}{the optimal} policy is the optimal action from all states $s \in S$.  Given that all rewards $R(s)$ must by definition fall within an $s \in S$, then we can conclude that every reward will be collected at least once.

For the remainder of the proof we make the simple observation that every point in the value function is by definition either a local maximum or not a local maximum at at given state $s_i$.

From this, if a state $s_i \in S$ is not a local maximum, then the optimal action $a^*$ will cause the state $s_i$ to be exited in favor of a next state $s_j$, \textcolor{accepted}{which is a neighbor state of $s_i$ with maximum value}. This process will continue until a local maximum is reached.

When a local maximum $s_M$ is reached, we necessarily then enter a minimum cycle $C^*(s_M)$ which by definition is a cycle where a primary reward and optionally one or more secondary rewards are collected infinitely.  Formally, a local maximum is thus defined as:  $V(s_i) \geq V(s) \forall s \in S | \delta(s, s_i) = 1, \forall s_i \in \mathcal{C}^*(s_M)$.

Therefore, a given reward must be collected once or infinitely, and only rewards in a minimum cycle (which is a local maximum) are collected infinitely.
\end{proof}

%
%
%


We note that the propagation operator $\mathcal{P}$ forms an exponential decay curve from the peak value.

The value functions for the baseline and delta baseline are simply the propagation of the peaks, $\mathcal{P}_{\mathcal{B}^i}(s)$ and $\mathcal{P}_{\Delta^d}(s)$ respectively. 

\begin{theorem}
The value function for the combined baseline is the sum of the baselines for each peak.

Given two reward sources $r_p$ at state $s_p$ and $r_s$ at state $s_s$, where $r_p \geq r_s$ and $r_s$ is within the minimum cycle of $r_p$, the value function for the combined peak is equal to the sum of the baselines of each peak:
\begin{equation}
\begin{split}
V(s) &= \mathcal{P}_{\mathcal{B}^p}(s) + \mathcal{P}_{\mathcal{B}^s}(s) \\
\end{split}
\label{eq:comb_value}
\end{equation}
\end{theorem}

\begin{proof}
  For any state $s \in \mathcal{S}$, we will show that the value function in \ref{eq:comb_value} satisfies the Bellman Optimality Equation.

For the case $s=s_p$, we have:
\begin{equation}
  V(s_p) = r_p + \gamma \max_a V(T(s,a))
\end{equation}
It should be noted that $\max_a V(T(s,a)) = V(s_s) = \gamma \bb^p(s) + \bb^s(s)$, since any other action will take the agent to state with distance 1 to $s_p$ and distance 2 to $s_s$, which will have value $\gamma \bb^p + \gamma^2 \bb^s$, which is less than $V(s_s)$. Thus we have:
\begin{equation}
\begin{split}
  V(s_p) &= r_p + \gamma (\gamma \bb^p + \bb^s)  \\
         &= \bb^p + \gamma \bb^s \\
  V(s) &= \mathcal{P}_{\bb^p}(s) + \gamma \mathcal{P}_{\bb^s}(s), \forall s \in S
\end{split}
\end{equation}
which is consistent with the value function in \ref{eq:comb_value}.

For the case $s = s_s$, it is similar to the case $s = s_p$.

For the case $s \neq s_p$ and $s\neq s_s$. We first note that for a 2D grid world MDP with two neighboring reward state $s_p, s_s$, the effect of an action is to lead the agent one step further from the one reward state, e.g. $s_p$ or one step nearer to this reward state. Assuming for our current state $s$, the distance from $s$ to $s_p$ , denoted as $\delta(s,s_p)$, is $n$. And the distance from $s$ to $s_s$, denoted as $\delta(s,s_s)$ is $n+1$ (it can also be $n-1$, and the proof would be similar). Then after one action, $\delta(s,s_p)=n-1$ or $\delta(s,s_p) = n+1$, and $\delta(s,s_s) = n$ or $\delta(s,s_s) = n+2$. Then according to Bellman Equation, we have
\begin{equation}
  V(s) = \gamma \max_a V(T(s,a))
\end{equation}
Since $\gamma^{n-1} \bb^p > \gamma^{n+1} \bb^p$ and $\gamma^{n} \bb^s > \gamma^{n+2} \bb^s$, the action that leads to $\delta(s,s_p) = n-1$ and $\delta(s,s_s) = n$ will be the optimal action. So we have
\begin{equation}
\begin{split}
  V(s_p) &= \gamma (\gamma^{n-1} \bb^p + \gamma^{n} \bb^s)	\\
   &= \gamma^{n} \bb^p + \gamma^{n+1} \bb^s	\\
   &= \gamma^{\delta(s,s_p)} \bb^p + \gamma^{\delta(s,s_s)} \bb^s \\
  V(s) &= \mathcal{P}_{\bb^p}(s) + \gamma \mathcal{P}_{\bb^s}(s), \forall s \in S
\end{split}
\end{equation}
which is consistent with the value function in \ref{eq:comb_value}.
\end{proof}

Thus we have identified the three possible ways that a reward in $S^+$ can be collected.  How do we then select between these alternatives in order to find the optimal value function $V^*$?

\subsubsection{Constructing Value Function}

Here we show that the optimal value function is formed from a combination of the baselines defined in the previous section.

\begin{definition}
Let $R = \{ r_1, r_2 ... r_N \}$ be the set of rewards sources in an MDP, and let $|R| = N$ be the number of reward sources.  Let the set of all possible baseline value functions be $\mathcal{P}_{\mathcal{B}} = \{ \mathcal{P}_{\mathcal{B}_1}, ... \mathcal{P}_{\mathcal{B}_N} \}$.  Similarly, let the set of all possible combined baseline value functions be $\mathcal{P}_{\Gamma} = \{ \mathcal{P}_{\Gamma_1}, ... \mathcal{P}_{\Gamma_N} \}$ and the set of all possible delta baseline value functions be $\mathcal{P}_{\Delta} = \{ \mathcal{P}_{\Delta_1}, ... \mathcal{P}_{\Delta_N} \}$.  Then let $\mathcal{M} = \mathscr{P} \left( \mathcal{P}_{\mathcal{B}} \cup \mathcal{P}_{\Gamma} \cup \mathcal{P}_{\Delta} \right)$ be the power set of all possible baselines and $M \subset \mathcal{M}$ be one such selected combination of baselines.  We denote a specific value function for a baseline with within $M$ as $M_i$.
\end{definition}

\begin{definition}
For a specific combination of baselines $M \in \mathcal{M}$, we define the value function $V^M$ as the maximum value over all value functions in $M$:
\begin{equation*}
V^M(s) = \max_{M_i \in M} M_i(s), \forall s \in S
\end{equation*}

We denote the set of all value functions formed by $\mathcal{M}$ as $V^{\mathcal{M}} = \{ V^M \}, \forall M \in \mathcal{M}$.
\end{definition}

\begin{definition}
We denote as $V^{\alpha}$ as the region between the optimal value function $V^*$ and the zero-function $V_{\emptyset}(s) = 0$.
\begin{equation*}
0 \leq V^{\alpha}(s) \leq V^*(s), \forall s \in S
\end{equation*}
\end{definition}

We pause now to consider these definitions and informally relate them to traditional well known intuitions between policies and value functions.  We note that traditionally $V^\pi$ is defined as the set of value functions formed by all possible policies $\pi \in \Pi$.  We also note that value iteration iteratively searches through a countably infinite set of functionals that estimate $V^*$, asymptotically approaching $V^*$, and that the set of such functions becomes finite when a stopping criterion such as the bellman residual is used.  We note that there are an uncountably infinite number of functions $f(s) \in V^{\alpha}$, many of which cannot be part of $V^\pi$ because no policy can generate these functions under the MDP.  

In general, a policy $\pi^M$ can be extracted from any value function $V^M \in V^{\mathcal{M}}$, and this $\pi^M$ is guaranteed to fall within $\Pi$ because $\Pi$ by definition contains all possible policies for the given MDP definition.

We can think of $V^M$ as considering a subset of the original MDP problem, where the state and action space are identical, but with a subset of the rewards.  Therefore, when a policy is extracted from $V^M$ and then applied to the full MDP formulation, a function in $V^\pi$ is generated.  Thus, generally, $V^\mathcal{M}$ lies outside of $V^\pi$.  However, $V^\mathcal{M}$ and $V^\pi$ both contain the optimal solution $V^*$ (which will be proven below) and thus the optimal solution within $V^\mathcal{M}$ is also an optimal solution within $V^\pi$.

\newcommand{\DrawBean}[6]{
 \tkzDefPoint(#1+#2*(0.4   ) ,#3*0.8   ){A}
 \tkzDefPoint(#1+#2*(1     ) ,#3*0.4   ){B}
 \tkzDefPoint(#1+#2*(1.8   ) ,#3*1.0   ){C}
 \tkzDefPoint(#1+#2*(2.3773) ,#3*1.4421){D}
 \tkzDefPoint(#1+#2*(2.6905) ,#3*2.1074){E}
 \tkzDefPoint(#1+#2*(2.3752) ,#3*2.8828){F}
 \tkzDefPoint(#1+#2*(1.4   ) ,#3*3.0   ){G} 
 \tkzDefPoint(#1+#2*(0.6   ) ,#3*2     ){H} 
 
 \tkzDefPoint(#1+#2*(1.0   ) ,#3*0.8   ){I} 
 \tkzDefPoint(#1+#2*(1.8   ) ,#3*1.6   ){J} 
 \tkzDefPoint(#1+#2*(1.2   ) ,#3*2.0   ){K} 
 \tkzDefPoint(#1+#2*(0.7   ) ,#3*1.3   ){L} 
 
 \tkzDefPoint(#1+#2*(1.4   ) ,#3*1.5   ){M} 
 \tkzDefPoint(#1+#2*(2.4   ) ,#3*2.1   ){N} 
 
 
 \tkzLabelPoint[above      ](I){#5}
 \tkzLabelPoint[above left](M){#6}
 \tkzLabelPoint[below left](N){#4}

 \node [color=red] at (M) {\textbullet};

 \draw plot [smooth cycle, tension=0.7] coordinates {
 (A) (B) (C) (D) (E) (F) (G) (H)
 };
 \draw plot [smooth cycle, tension=0.7] coordinates {
 (I) (J) (K) (L) 
 };
}

\newcommand{\DrawOther}[6]{
 \tkzDefPoint(#1+#2*(0.4   ) ,#3*0.8   ){A}
 \tkzDefPoint(#1+#2*(1     ) ,#3*0.4   ){B}
 \tkzDefPoint(#1+#2*(1.8   ) ,#3*1.0   ){C}
 \tkzDefPoint(#1+#2*(2.3773) ,#3*1.4421){D}
 \tkzDefPoint(#1+#2*(2.6905) ,#3*2.1074){E}
 \tkzDefPoint(#1+#2*(2.3752) ,#3*2.8828){F}
 \tkzDefPoint(#1+#2*(1.4   ) ,#3*3.0   ){G} 
 \tkzDefPoint(#1+#2*(0.6   ) ,#3*2     ){H} 
 
 \tkzDefPoint(#1+#2*(0.2   ) ,#3*0.8   ){I} 
 \tkzDefPoint(#1+#2*(1.0   ) ,#3*1.4   ){J} 
 \tkzDefPoint(#1+#2*(0.3   ) ,#3*2.5   ){K} 
 \tkzDefPoint(#1+#2*(-.4   ) ,#3*1.3   ){L} 
 
 \tkzDefPoint(#1+#2*(0.6   ) ,#3*1.4   ){M} 
 \tkzDefPoint(#1+#2*(2.4   ) ,#3*2.1   ){N} 
 
 \tkzDefPoint(#1+#2*(-.9   ) ,#3*0.8   ){O}
 \tkzDefPoint(#1+#2*(0.8   ) ,#3*0.0   ){P}
 \tkzDefPoint(#1+#2*(1.8   ) ,#3*0.2   ){Q}
 \tkzDefPoint(#1+#2*(2.4773) ,#3*1.0421){R}
 \tkzDefPoint(#1+#2*(2.9905) ,#3*2.1074){S}
 \tkzDefPoint(#1+#2*(2.8752) ,#3*3.0828){T}
 \tkzDefPoint(#1+#2*(1.4   ) ,#3*3.4   ){U} 
 \tkzDefPoint(#1+#2*(0.3   ) ,#3*3     ){V} 
 
 
 \tkzLabelPoint[below right](L){#5}
 \tkzLabelPoint[above right](M){#6}
 \tkzLabelPoint[below left](N){#4}
 \tkzLabelPoint[below ](U){$V^\alpha$}

 \node [color=red] at (M) {\textbullet};

 \draw plot [smooth cycle, tension=0.7] coordinates {
 (A) (B) (C) (D) (E) (F) (G) (H)
 };
 \draw plot [smooth cycle, tension=0.7] coordinates {
 (I) (J) (K) (L) 
 };
 \draw plot [smooth cycle, tension=0.7] coordinates {
 (O) (P) (Q) (R) (S) (T) (U) (V)
 };
}

\def\oa{1}
\def\ob{9}
\begin{figure}[H]
\centering
\begin{tikzpicture}[label]
 \DrawBean{\oa}{2}{2}{$\pi$}{$\pi^{\mathcal{M}}$}{$\pi^*$}
 \DrawOther{\ob}{2}{2}{$V^\pi$}{$V^{\mathcal{M}}$}{$V^*$}
\end{tikzpicture}
\caption{Depiction of the relationship between policy, value function, and optimal solution for $V^{\mathcal{M}}$}
\end{figure}
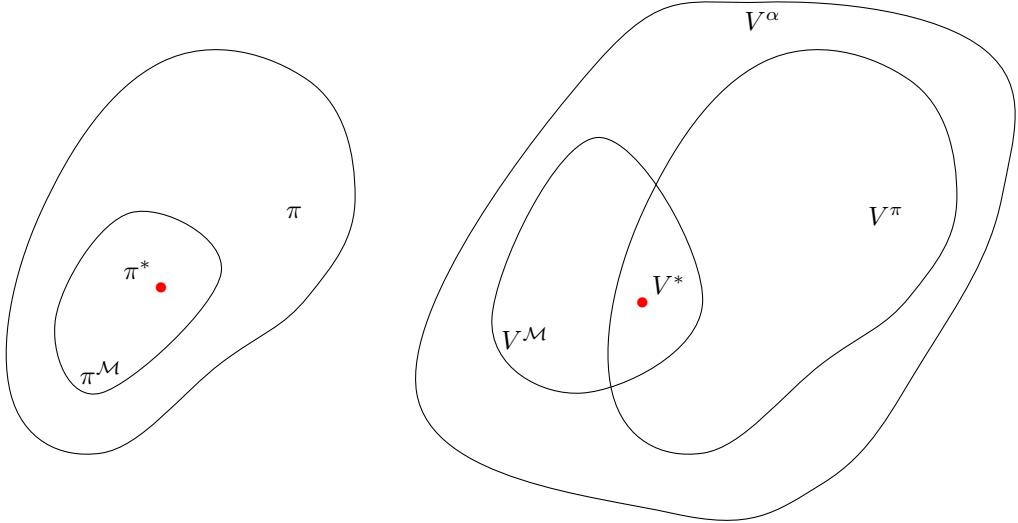
               
\begin{theorem}
At every state in $S^+$, there is a value function in $V^{\mathcal{M}}$ that is at least as large as any in $V^\pi$:

\begin{equation*}
\forall s \in S^+, \max_{M \in \mathcal{M}} V^M(s) \geq \max_{\pi \in \Pi} V^\pi(s)
\end{equation*}
\end{theorem}

\begin{proof}
From the previous section, we have identified the three possible ways that a reward $r_i$ at state $s_i \in S^+$ can be collected: the baseline $\mathcal{P}_{\mathcal{B}^i}(s)$ with peak value $\mathcal{B}^i$, the combined baseline $\mathcal{P}_{\Gamma^i}(s)$ with peak value $\Gamma^i$, and the delta baseline $\mathcal{P}_{\Delta^i}(s)$ with peak value $\Delta^i$ which we will denote as the set $\mathcal{M}_{s_i} \subset \mathcal{M}$.  

Given that the set $\mathcal{M}_{s_i}$ represents the values functions that can possibly result at state $s_i$, then there must be a maximum among them such that $\exists m_{max} \in \mathcal{M}_{s_i} | \forall m \in \mathcal{M}_{s_i}, \mathcal{M}_{s_i}(m_{max}) \geq \mathcal{M}_{s_i}(m)$.  The maximum possible value at $s_i$ is then defined by $\mathcal{M}_{s_i}(m_{max})$ and is thus equal to $V^*(s_i)$.  Given then that $V^*(s_i)$ is an upper bound at $s_i$ for both $V^\pi(s_i)$ and $\mathcal{M}_{s_i}(m_{max})$:

\begin{equation*}
\begin{split}
\forall s_i \in S^+, \max_{M \in \mathcal{M}_{s_i}} V^M &\geq \max_{\pi \in \Pi} V^\pi(s_i)
\end{split}
\end{equation*}

\end{proof}

We can extend the above to cover the entire value function:

\begin{theorem}
At every state in the whole of $S$, there is a value function in $V^{\mathcal{M}}$ that is at least as large as any in $V^\pi$:
\begin{equation*}
\forall s \in S, \max_{M \in \mathcal{M}} V^{M}(s) \geq \max_{\pi \in \Pi} V^\pi(s)
\end{equation*}
\end{theorem}

\begin{proof}
Given that we now know from the previous theorem the maximum value for all states in $S^+$, then from Theorem \ref{thm:determinism} we can say that the value of all states in $S$ are known.

To show that the values in $S^Z$ are maximum, we recall that the propagation operator $\mathcal{P}$ forms an exponential decay curve from the peak value $v_p$ at state $s_p$ of the form:

\begin{equation*}
\forall s \in S, \mathcal{P}_p(s) = v_p \times \gamma^{\delta(s,s_p)},
\end{equation*}

where $\delta(s,s_p)$ is the distance from $s$ to the peak at $s_p$.

The exponential decay curve has the property that at a given state $s_i$, if two peak values $p_1$ and $p_2$ are considered, and supposing that $p_1 \geq p_2$, then $\forall s \in S, \mathcal{P}_{p_1}(s) \geq \mathcal{P}_{p_2}(s)$. Thus, if we know the peak of the value functions in $s \in S^+$ are maximum, then we know that the corresponding exponential decay curve is maximum in $S^Z$ as well.
\end{proof}

\begin{theorem}
The optimal value function $V^*$ lies within $V^{\mathcal{M}}$ and is in fact the element-wise maximum of all value functions in $V^{\mathcal{M}}$.
\begin{equation*}
\forall s \in S, V^*(s) = \max_{M \in \mathcal{M}} V^M(s)
\end{equation*}
\end{theorem}

\begin{proof}
From the above proofs, we know that at any state $s \in S^+$, the maximum possible value is $V^{max}(s) = \max_{M \in \mathcal{M}} V^M(s)$, and we know that the states in $S^Z$ can be determined by a peak in $S^+$.  However, there are multiple such peaks in $S^+$ which might determine the value of a given state $s_z \in S^Z$.

Recalling that the optimal value function is the maximum possible value at every state $s \in S$, and therefore that it is the maximum possible value at every state $s_z \in S^Z$, it is clear then that the maximum value at $s_z$ must be the maximum of all possible value functions in $M$ evaluated at $s_z$.

\begin{equation*}
\forall s_z \in S^Z, V^*(s_z) = \max_{M \in \mathcal{M}} M(s_z)
\end{equation*}

Given that $V^{M}(s) = \max_{M_i \in M} M_i(s), \forall s \in S$, this is equivalent to:
\begin{equation*}
\forall s_z \in S^Z, V^*(s_z) = \max_{M \in \mathcal{M}} V^M(s_z)
\end{equation*}

Thus we now know the maximum value at every state in both $S^Z$ and $S^+$, and therefore $S$ as a whole:

\begin{equation*}
\forall s \in S, V^*(s) = \max_{M \in \mathcal{M}} V^{M}(s)
\end{equation*}
\end{proof}

We may therefore conclude that the optimal value function $V^*$ is the max over each state $s \in S$ of the value function from the possible combinations of the peaks in $\mathcal{M}$.  This forms the core of the algorithm and completes the proof that the algorithm calculates the optimal value function $V^*$.

\end{document}